\documentclass{article}

\usepackage[numbers, square, comma, sort&compress]{natbib}
\usepackage[preprint]{neurips_2026}
\usepackage{hyphenat}
\setcitestyle{square, numbers, comma, sort&compress}
\usepackage[utf8]{inputenc} 
\usepackage[T1]{fontenc}    
\usepackage{hyperref}       
\usepackage{url}            
\usepackage{booktabs}       
\usepackage{amsfonts}       
\usepackage{nicefrac}       
\usepackage{microtype}      
\usepackage{xcolor}         
\usepackage{bbm}
\usepackage{dsfont}
\usepackage{amsmath}
\usepackage{enumitem}
\usepackage{amssymb,amsthm}
\usepackage{mathrsfs}
\usepackage{graphicx}
\usepackage{bbm}
\usepackage[capitalise]{cleveref}

\allowdisplaybreaks

\newtheorem{theorem}{Theorem}[section]
\newtheorem{condition}[theorem]{Condition}

\newtheorem{corollary}[theorem]{Corollary}
\newtheorem{proposition}[theorem]{Proposition}

\theoremstyle{definition}

\theoremstyle{remark}
\newtheorem{remark}[theorem]{Remark}

\def\beqlb{\begin{eqnarray}}\def\eeqlb{\end{eqnarray}}
\def\beqnn{\begin{eqnarray*}}\def\eeqnn{\end{eqnarray*}}

\def\mbf{\mathbf}

\def\itGamma{{\it\Gamma}}

\def\ar{\!\!\!&}

\title{Matrix-Decoupled Concentration for Autoregressive Sequences: Dimension-Free Guarantees \\  for Sparse Long-Context Rewards}

\author{%
 Pei-Sen Li \\
  School of Mathematics and Statistics, Beijing Institute of Technology, Beijing 100872, China \\
  \texttt{peisenli@bit.edu.cn} \\
}

\begin{document}

\maketitle
\begin{abstract}
Sequence-level evaluations in autoregressive Large Language Models (LLMs) rely on highly dependent token generation. Establishing tight concentration bounds for these processes remains a challenge due to two fundamental bottlenecks in existing frameworks: (i) classical inequalities typically separate dependency structures from target sensitivities, leading to a scalar collapse that inflates the variance proxy to a suboptimal $\mathcal{O}(N)$ for sparse terminal rewards; (ii) conversely, while certain spatial methods achieve tighter bounds, they lack the strictly causal filtration required by sequential generation, rendering them inapplicable to the autoregressive setting. To resolve both bottlenecks, we establish a sharp McDiarmid-type inequality for dependent sequences, governed strictly by the exact matrix-vector multiplication of the causal dependency resolvent and the target sensitivity vector. This Matrix-Decoupled Concentration (MDC) framework natively recovers optimal constants for Markov chains and exploits directed $d$-separation to yield order-optimal bounds for causal trees. Crucially, by exactly preserving the coordinate-wise sparsity of rewards within a strictly causal framework, MDC mathematically prevents scalar collapse, guaranteeing a dimension-free $\mathcal{O}(1)$ variance proxy and providing a rigorous mathematical justification for the stability of long-context reasoning.
\end{abstract}

\section{Introduction}
\label{sec:intro}
The theoretical foundation of modern machine learning, particularly in Reinforcement Learning from Human Feedback (RLHF) \cite{christiano2017deep, ouyang2022training, schulman2017proximal}, relies heavily on sequence-level evaluations. This includes modern alignment paradigms such as Outcome-based Reward Models (ORM) and Direct Preference Optimization (DPO) \cite{rafailov2023direct, lightman2023lets}. In autoregressive Large Language Models (LLMs), a model generates a trajectory of categorical tokens $X_1, \dots, X_N$, and a target function $f(X_1, \dots, X_N)$ assigns a global scalar score. Establishing finite-sample deviation bounds for these dependent generative processes is a fundamental challenge, typically expressed as tail bounds of the form:
\begin{equation}
\label{eq:tail_bound}
\mathbb{P} \left( \left| f(X_1, \dots, X_N) - \mathbb{E}[f(X_1, \dots, X_N)] \right| \ge t \right) \le \delta(t, N).
\end{equation}
Crucially, the optimal decay rate of the failure probability $\delta$ is jointly governed by the \textbf{sensitivity vector} of the target function $f$ and the \textbf{dependency structure} among the variables $X_1, \dots, X_N$.

Originating from classical bounds for independent variables \cite{hoeffding1963probability, mcdiarmid1989method}, the realm of McDiarmid-type inequalities for dependent sequences spans several mathematical paradigms (see, e.g., \cite{van2019concentration}). These include macroscopic mixing coefficients \cite{rio2000hoeffding, dedecker2007weak, klochkov2020uniform}, transportation-cost inequalities \cite{marton1996bounding, samson2000concentration}, spectral methods \cite{paulin2015concentration}, spatial coupling and Stein's method \cite{chazottes2007concentration, chatterjee2007stein}, and martingale-based graph or matrix methods \cite{janson2004large, kontorovich2008concentration, minsker2023matrix, howard2020time}. However, the non-Markovian, autoregressive generation in Large Language Models (LLMs) introduces profound structural mismatches for these existing frameworks, rendering them either suboptimal or inapplicable due to two distinct mathematical bottlenecks.

The first bottleneck is a phenomenon we formalize as \textbf{scalar collapse}. To obtain computable bounds, a majority of classical frameworks mathematically separate the dependency structure from the sensitivity vector. For instance, martingale graph methods \cite{kontorovich2008concentration, minsker2023matrix} isolate the interaction matrix via sub-multiplicative norm relaxations (e.g., bounding the variance proxy by proportional terms of $\|\itGamma\|_\infty^2 \|\mathbf{c}\|_\infty^2$). Similarly, transportation-cost approaches \cite{marton1996bounding} and spectral methods \cite{paulin2015concentration} condense temporal interactions into macroscopic scalar summaries, such as global contraction rates. Crucially, when evaluating sparse sequence-level rewards in LLMs, only the terminal token's sensitivity is non-zero ($c_N > 0$). In this context, extracting the sensitivity vector via its maximum norm ($\|\mathbf{c}\|_\infty$) imposes a uniform penalty across all coordinates. Factoring out these macroscopic scalars completely discards the topological sparsity inherent in the objective, artificially inflating the theoretical bound to an $\mathcal{O}(N)$ scale. 

The second bottleneck arises from \textbf{causal structure mismatch}. While spatial coupling \cite{chazottes2007concentration} and Stein's method \cite{chatterjee2007stein} bypass scalar collapse to achieve optimal $\mathcal{O}(N)$ concentration on undirected graphical models, they introduce a different fatal flaw. These spatial methods rely on differing boundary conditions on a global Gibbs measure, an approach that inherently requires simultaneous conditioning on the entire graphical structure. This global conditioning strictly violates the non-anticipative, sequential arrow of time required by autoregressive decoding. Consequently, despite empirical successes in scaling context windows to millions of tokens \cite{gemini2024gemini, openai2023gpt4, su2024roformer}, existing theory either implies an artificial $\mathcal{O}(N)$ instability or requires inapplicable global conditioning. 

To resolve both bottlenecks, we introduce the \textbf{Matrix-Decoupled Concentration (MDC)} framework. Instead of compressing the dependency structure into a macroscopic scalar, MDC explicitly encodes it as a causal interdependence matrix $H$. The variance proxy is then governed strictly by the exact matrix-vector multiplication $\|(I-H)^{-1} \mathbf{c}\|_2^2$, which retains significantly more structural information about the interaction between dependencies and sensitivities. By structurally isolating the non-zero terminal sensitivities, this algebraic formulation prevents the global accumulation of intermediate errors, yielding an order-optimal, dimension-free $\mathcal{O}(1)$ variance proxy for sparse rewards and analytically bypassing the $\mathcal{O}(N)$ inflation. Crucially, because its algebraic construction strictly respects the causal ordering of sequential generation, MDC provides a natively causal and rigorous mathematical foundation for evaluating autoregressive LLMs and broader non-Markovian sequential generation paradigms.

Our main contributions are summarized as follows:
\begin{itemize}[leftmargin=*]
    \item \textbf{The MDC Framework:} We propose Matrix-Decoupled Concentration (MDC) to encode the conditional dependency structure into a TV matrix $H$. We establish a McDiarmid-type inequality where variance proxies are determined by the exact matrix-vector multiplication of the causal resolvent $(I-H)^{-1}$ and the sensitivity vector $\mbf{c}$, preventing scalar collapse.
    \item \textbf{Generality and Sharpness:} We demonstrate that MDC natively recovers the optimal transport constants for homogeneous Markov chains \cite{marton1996bounding} and exploits graphical $d$-separation to establish order-optimal bounds for causal trees.
    \item \textbf{Dimension-Free Stability for LLMs:} We prove that MDC explicitly preserves the coordinate-wise sparsity of the sensitivity vector. For sparse rewards, the exact matrix-vector multiplication eliminates the $\mathcal{O}(N)$ inflation of classical bounds, guaranteeing an $\mathcal{O}(1)$ proxy for long-context generation.
\end{itemize}

\section{Problem Setup and Preliminaries}
\label{sec:setup}

Let $\mathcal{A}$ be an arbitrary finite set. Consider the canonical path space $\mathcal{A}^N$, equipped with the discrete $\sigma$-algebra $\mathcal{F} = 2^{\mathcal{A}^N}$. On this space, we define the sequence of random variables $\mathbf{X} := \{X_i\}_{i=1}^N$ as the coordinate projections, such that $X_i(x) = x_i$ for any trajectory $x = (x_1, \dots, x_N) \in \mathcal{A}^N$. For integers $1 \le a \le b \le N$, we use the shorthand $X_{a:b} := (X_a, \dots, X_b)$ and $x_{a:b} := (x_a, \dots, x_b)$. Our goal is to establish a sharp concentration bound for a measurable target function $f: \mathcal{A}^N \to \mathbb{R}$, without assuming mutual independence of $\mathbf{X}$.

To formally specify the dependence structure of $\mbf{X}$, we employ a bottom-up construction via a family of transition kernels $\{p_i\}_{i=1}^N$. For every step $i \in \{1,\dots,N\}$ and every history $x_{1:i-1} \in \mathcal{A}^{i-1}$, let $p_i(\,\cdot \mid x_{1:i-1})$ be a probability measure on $\mathcal{A}$. By the chain rule of probability, these localized transition kernels uniquely induce a global joint probability measure $\mathbb{P}$ on the path space $(\mathcal{A}^N, \mathcal{F})$, satisfying:
\beqnn
\mathbb{P}(X_{1:N} = x_{1:N}) = \prod_{i=1}^N p_i(x_i \mid x_{1:i-1}).
\eeqnn

We explicitly formalize the mapping of this mathematical structure to autoregressive Large Language Models (LLMs): the abstract state space $\mathcal{A}$ corresponds to the token vocabulary $\mathcal{V}$, and $p_i$ represents the conditional probability measure defined by the autoregressive prediction head for any arbitrary prefix. The joint distribution $\mathbb{P}$ therefore encodes the exact directed causal dependencies across generative steps.

To quantify the state-to-state conditional dependencies generated by these transition kernels, we rely on the Total Variation (TV) distance. For a finite set $\mathcal{A}$, the TV distance between measures $\mu$ and $\nu$ is:
\beqnn
d_{\mathrm{TV}}(\mu, \nu) = \frac{1}{2} \sum_{a \in \mathcal{A}} |\mu(a) - \nu(a)| = \max_{B \subseteq \mathcal{A}} |\mu(B) - \nu(B)|.
\eeqnn

Before bounding the stochastic dependency, we isolate the deterministic sensitivity of the target function $f$ using a coordinate-wise Lipschitz condition.

\begin{condition}[Generalized Lipschitz Vector]\label{cond:lipschitz_general}
The function $f: \mathcal{A}^N \to \mathbb{R}$ admits a sensitivity vector $\mathbf{c} = (c_1, \dots, c_N)^\top$ such that for any $x, y \in \mathcal{A}^N$,
$
|f(x) - f(y)| \le \sum_{j=1}^N c_j \mathbbm{1}_{\{x_j \neq y_j\}}.
$
\end{condition}

\textbf{Sparse Sequence-Level Reward.} In autoregressive generation where the evaluation metric depends exclusively on the final output (e.g., assessing only the terminal state $X_N$), altering intermediate coordinates does not change the objective value. This yields $c_j = 0$ for all $j < N$, and the sensitivity vector strictly reduces to $\mathbf{c} = (0, \dots, 0, c_N)^\top$.

\section{The Matrix-Decoupled Concentration Framework}
\label{sec:framework}

We now formalize the Matrix-Decoupled Concentration (MDC) framework. To prevent the scalar collapse discussed in Section \ref{sec:intro}, we encode the step-wise conditional dependencies into an interdependence matrix.

\subsection{The Causal Interdependence Matrix and Resolvent}

We define the \emph{causal interdependence matrix} $H \in \mathbb{R}^{N \times N}$, where each entry $H_{i,j}$ bounds the directed influence of step $i$ on the conditional distribution of step $j$. Because the transition kernel $p_j$ depends strictly on the causal history $x_{1:j-1}$, any perturbation at coordinate $i \ge j$ has zero influence on the conditional distribution. Consequently, we strictly set $H_{i,j} = 0$ for all $i \ge j$. For strictly upper-triangular entries ($1 \le i < j \le N$), $H_{i,j}$ bounds the maximum total variation distance between the transition kernels of $X_j$ caused by altering the single state of $X_i$:
\beqlb\label{defH}
H_{i,j} := \max_{\substack{\mathbf{z} \in \mathcal{A}^{i-1}\\ x, x' \in \mathcal{A}\\ \mathbf{w} \in \mathcal{A}^{j-i-1}}} 
d_{\mathrm{TV}}\Big( p_j(\cdot \mid \mathbf{z}, x, \mathbf{w}) \, , \, p_j(\cdot \mid \mathbf{z}, x', \mathbf{w}) \Big).
\eeqlb
Evaluating this supremum by explicitly conditioning on the intermediate trajectory $\mathbf{w}$ mathematically isolates the direct causal influence of step $i$ on the transition probability at step $j$. Conceptually, a non-zero entry $H_{i,j} > 0$ indicates a direct causal edge from $i$ to $j$. Because the intermediate trajectory $\mathbf{w}$ fixes the intervening states, any structural conditional independence (e.g., the Markov property) ensures that perturbations at step $i$ do not alter the transition probability at step $j$. This yields exactly $H_{i,j} = 0$ whenever the two steps lack a direct causal connection.

Since $H$ is strictly upper-triangular and entry-wise non-negative, it is nilpotent of degree $N$ (i.e., $H^N = \mathbf{0}$). This algebraic property guarantees that the matrix $I-H$ is invertible, and its inverse is given exactly by the finite, entry-wise non-negative Neumann series, which we term the \emph{causal resolvent matrix}:
\beqlb\label{eq:gamma_def}
\itGamma := (I - H)^{-1} = \sum_{r=0}^{N-1} H^r.
\eeqlb

\subsection{Matrix-Decoupled Concentration Theorem}

We present our main theorem, which bounds the sequential divergence purely through the exact matrix-vector multiplication between the causal resolvent $\itGamma$ and the sensitivity vector $\mathbf{c}$.

\begin{theorem}[Matrix-Decoupled Concentration]\label{thm:general_matrix_decoupling}
Suppose the sequence $\mathbf{X}:=\{X_i\}_{i=1}^N$ takes values in a finite set $\mathcal{A}$. For any target function $f$ satisfying Condition \ref{cond:lipschitz_general} with sensitivity vector $\mathbf{c}$, we have the explicit concentration inequality:
\beqlb\label{eq:delta_matrix_bound_general}
\mathbb{P}\bigl(|f(\mathbf{X})-\mathbb{E}[f(\mathbf{X})]|\ge t\bigr)
\le
2\exp\left(-\frac{2t^{2}}{\|\itGamma \mathbf{c}\|_2^2}\right),
\eeqlb
where $\|\itGamma \mathbf{c}\|_2^2 := \sum_{k=1}^N (\itGamma \mathbf{c})_k^2$ is the squared vector $\ell_2$-norm.
\end{theorem}

\begin{proof}[Proof Sketch of Theorem \ref{thm:general_matrix_decoupling}]
We provide a brief outline of the key algebraic steps below; the complete measure-theoretic derivation is deferred to Appendix \ref{sec:appendix_proof}.

\emph{Step 1: Martingale Decomposition.} Let $\Delta_k = M_k - M_{k-1}$ be the Doob martingale difference sequence for $f(\mathbf{X})$ with respect to the natural filtration. For any fixed conditioning prefix $\mathbf{x} \in \mathcal{A}^{k-1}$, the conditional range (oscillation) of $\Delta_k$, denoted as $\delta_k(\mathbf{x})$, is evaluated by the maximum divergence between expected future trajectories conditioned on differing realizations $x$ and $x'$ at step $k$.

\emph{Step 2: Vectorized Path Coupling.} We construct a joint probability coupling $(\mathbf{Y}, \mathbf{Z})$ such that the prefixes are deterministically fixed to $Y_{1:k} = (\mathbf{x}, x)$ and $Z_{1:k} = (\mathbf{x}, x')$. Let $v_i := \hat{\mathbb{P}}(Y_i \neq Z_i)$ be the marginal coupling discrepancy probability. By evaluating the telescoping sum of the total variation distance over intermediate trajectories, the discrepancy satisfies a recursive bound constrained by the interdependence matrix $H$:
\beqnn
v_j \le \sum_{m=k}^{j-1} H_{m,j} v_m, \quad \forall j > k.
\eeqnn
Defining the discrepancy vector $\mathbf{v} := (0, \dots, 0, 1, v_{k+1}, \dots, v_N)^\top$, this recursive system is expressed in matrix-vector form as $(I - H^\top)\mathbf{v} \le \mathbf{e}_k$, where $\mathbf{e}_k$ is the $k$-th standard basis vector.

\emph{Step 3: Algebraic Resolvent.} Because $H$ is strictly upper-triangular and entry-wise non-negative, the transposed resolvent matrix $\itGamma^\top = (I - H^\top)^{-1} = \sum_{r=0}^{N-1} (H^\top)^r$ is also entry-wise non-negative. Left-multiplying the inequality by $\itGamma^\top$ strictly preserves the inequality direction, yielding $\mathbf{v} \le \itGamma^\top \mathbf{e}_k$. Applying the Lipschitz Condition \ref{cond:lipschitz_general} connects the probability bounds exactly to the target function's sensitivity vector $\mathbf{c}$:
\beqnn
\delta_k(\mathbf{x}) \le \mathbf{c}^\top \mathbf{v} \le \mathbf{c}^\top (\itGamma^\top \mathbf{e}_k) = (\itGamma \mathbf{c})^\top \mathbf{e}_k = (\itGamma \mathbf{c})_k.
\eeqnn
This algebra guarantees that the conditional range of the zero-mean martingale difference $\Delta_k$ is strictly bounded by the scalar length $(\itGamma \mathbf{c})_k$. Applying Hoeffding's Lemma  \cite{hoeffding1963probability} to $\Delta_k$ and optimizing the exponential moments via Markov's inequality yields the stated theorem.
\end{proof}

While the exact matrix-vector multiplication $\itGamma \mathbf{c}$ natively preserves coordinate sparsity, extracting a scalar spectral decay coefficient directly connects our formulation to classical uniform bounds.

\begin{corollary}[Spectral Decay Coefficient]
\label{cor:general_kappa}
Assume all conditions in Theorem \ref{thm:general_matrix_decoupling} hold. Define the spectral decay coefficient of the system as $\kappa := \|\itGamma\|_2^{-2}$. Then for any $t > 0$:
\beqnn
\mathbb{P}\Big( |f(\mathbf{X}) - \mathbb{E}[f(\mathbf{X})]| \ge t \Big) \le 2\exp\left( - \frac{2 \kappa t^2}{\|\mathbf{c}\|_2^2} \right).
\eeqnn
Furthermore, if $H_{i,j} \le \varphi_{j-i}$ for all $1 \le i < j \le N$ with a bounded sum $S = \sum_{k=1}^\infty \varphi_k < 1$, then $\kappa \ge (1 - S)^2$. 
\end{corollary}
\begin{proof}
The formal proof, utilizing Schur's test and the discrete Neumann series, is deferred to Appendix \ref{sec:appendix_proof}.
\end{proof}

For independent sequences, the transition kernels are identical regardless of history, ensuring $H = \mathbf{0}$. Consequently, the causal resolvent  $\itGamma$ reduces to the identity matrix, and the variance proxy strictly evaluates to $\|\mathbf{c}\|_2^2$, exactly recovering the classical McDiarmid's inequality without structural loss.

\section{Theoretical Comparisons and Structural Optimality}
\label{sec:comparison}

To demonstrate the strict sharpness of the MDC framework, we evaluate it against established concentration inequalities for two contrasting generative topologies: time-homogeneous Markov chains and directed causal trees.

\subsection{Recovering Optimal Constants for Markov Chains}
\label{subsec:markov_chains}
Despite its profound applicability to complex non-Markovian dependencies, we first apply the MDC framework to time-homogeneous Markov chains to verify its correspondence with classical optimal transport bounds.

Consider a time-homogeneous Markov chain $\mathbf{X} = \{X_i\}_{i=1}^N$ taking values in $\mathcal{A}$. Under the Markov property, the general sequence transition kernel $p_j(\cdot \mid x_{1:j-1})$ reduces to the time-invariant one-step transition probability $P(\cdot \mid x_{j-1})$. For any $j > i+1$, the intermediate trajectory $\mathbf{w} = x_{i+1:j-1}$ explicitly contains the conditioning state $x_{j-1}$. By the Markov property, for any prefix $\mathbf{z} \in \mathcal{A}^{i-1}$ and any perturbed states $x, x' \in \mathcal{A}$ at step $i$, we have the identity $p_j(\cdot \mid \mathbf{z}, x, \mathbf{w}) = P(\cdot \mid x_{j-1}) = p_j(\cdot \mid \mathbf{z}, x', \mathbf{w})$. Substituting this into Definition \eqref{defH} immediately yields $H_{i,j} = 0$ for all $j > i+1$.

We assume this chain possesses a uniform Dobrushin contraction coefficient $\alpha \in (0, 1)$, defined by:
\begin{equation}
\label{eq:dobrushin_alpha}
\alpha := \max_{x, x' \in \mathcal{A}} d_{\mathrm{TV}}\Big(P(\cdot \mid x) \, , \, P(\cdot \mid x')\Big).
\end{equation}
Consequently, the entries of $H$ satisfy $H_{i, i+1} \le \alpha$, with $H_{i,j} = 0$ otherwise. This nilpotency enables the explicit computation of the resolvent, establishing the following proposition which recovers the constants derived via transportation-cost inequalities \cite{marton1996bounding}:

\begin{proposition}[Bounds for Contracting Markov Chains]
\label{prop:markov_improvement}
Let $\mathbf{X}$ be a time-homogeneous Markov chain with Dobrushin contraction coefficient $\alpha \in (0, 1)$. For any target function $f$ satisfying Condition \ref{cond:lipschitz_general} with sensitivity vector $\mathbf{c}$, we have:
\beqlb\label{eq:markov_bound}
\mathbb{P}\Big(|f(\mathbf{X}) - \mathbb{E}[f(\mathbf{X})]| \ge t\Big) \le 2\exp\left(-\frac{2t^2(1-\alpha)^2}{\|\mathbf{c}\|_2^2}\right).
\eeqlb
\end{proposition}
\begin{proof} Because $H$ is a strictly upper-triangular matrix with non-zero entries existing exclusively on the first superdiagonal (bounded by $\alpha$), both the maximum absolute row sum and column sum of $H$ are bounded by $\alpha$. By Schur's test, the operator spectral norm satisfies $\|H\|_2 \le \alpha$. Consequently, the resolvent matrix $\itGamma = (I-H)^{-1}$ has a spectral norm bounded by $\|\itGamma\|_2 \le (1-\alpha)^{-1}$. This mathematically guarantees the vector $\ell_2$-norm inequality $\|\itGamma \mathbf{c}\|_2^2 \le \|\itGamma\|_2^2 \|\mathbf{c}\|_2^2 \le (1-\alpha)^{-2} \|\mathbf{c}\|_2^2$. Substituting this explicit variance proxy into Theorem \ref{thm:general_matrix_decoupling} yields the desired bound.
\end{proof}

\begin{remark}[Comparison with Classical Frameworks]
The MDC framework provides an analytical mechanism that bypasses specific structural bottlenecks present in two classical paradigms for dependent sequences:
\begin{enumerate}[leftmargin=*]
    \item \emph{Metric Conversion in Optimal Transport:} While Marton \cite{marton1996bounding} established the optimal $(1-\alpha)^{-2}$ variance multiplier using $L_1$ transportation costs, extending optimal transport to $L_2$ geometries to cover general Lipschitz functions (e.g., Samson \cite{samson2000concentration}) requires bounding the squared Wasserstein distance by the Total Variation distance. This metric conversion forces the one-step contraction $\alpha$ to appear as $\sqrt{\alpha}$, modifying the variance multiplier to $(1-\sqrt{\alpha})^{-2}$ \cite[Eq.~(2.8)]{samson2000concentration}. MDC avoids metric conversion by operating strictly on discrete probability differences via the TV matrix $H$, yielding the $(1-\alpha)^{-2}$ constant without requiring measure tensorization or strict convexity restrictions.
    \item \emph{Unconditional Matrices in Martingale Graphs:} Frameworks such as \cite{kontorovich2008concentration} evaluate the influence of $X_i$ on $X_j$ without conditioning on the intervening sequence $X_{i+1:j-1}$. For a Markov chain, failing to apply this conditional separation yields a dense upper-triangular matrix $\Delta_{i,j} = \alpha^{j-i}$. The infinity norm evaluates to $\|\Delta\|_\infty = \frac{\alpha}{1-\alpha}$, forcing a resolvent bound that scales with $(1 - \|\Delta\|_\infty)^{-2} = \left(\frac{1-\alpha}{1-2\alpha}\right)^2$. This formulation diverges mathematically for any $\alpha \ge 1/2$. Furthermore, decoupling the matrix via the sub-multiplicative inequality $\|\Delta\|_\infty^2 \|\mathbf{c}\|_\infty^2$ forces a scalar collapse. MDC rectifies both issues: conditioning strictly on one-step transitions truncates the matrix ($H_{i, i+1} \le \alpha$), guaranteeing mathematical convergence for all $\alpha \in (0,1)$, while evaluating $\|\itGamma \mathbf{c}\|_2^2$ exactly preserves the coordinate-wise sparsity of $\mathbf{c}$.
\end{enumerate}
\end{remark}

\subsection{Topological Sparsity in Causal Trees}
\label{subsec:causal_trees}

We apply the Matrix-Decoupled Concentration (MDC) framework to a sequence $\mathbf{X} = \{X_i\}_{i=1}^N$ generated by a directed causal tree. This structure demonstrates how MDC mathematically preserves arbitrary dependence structures within a standard sequential filtration.
Establishing optimal concentration bounds for graph-structured data under 
sequential generation is a recognized mathematical challenge. Existing literature 
handles graph-structured dependencies through two primary paradigms, both of 
which introduce structural limitations for autoregressive generation:

\begin{enumerate}[leftmargin=*]
    \item \textbf{One-Dimensional Linear Filtrations:} Sequential martingale 
    methods \cite{kontorovich2008concentration, samson2000concentration, 
    boucheron2013concentration} maintain a valid causal filtration but require 
    mapping the graph nodes into a one-dimensional sequence. Imposing a total 
    order on a tree structure forces variables to condition on topologically 
    unrelated parallel branches (a structural bottleneck analogous to chromatic 
    number dependencies in \cite{janson2004large}). This structural mismatch 
    can introduce pseudo-dependencies that inflate the macroscopic mixing 
    coefficients, potentially degrading the variance proxy from $\mathcal{O}(N)$ 
    to $\mathcal{O}(N^2)$ or causing the bounds to diverge.
    
    \item \textbf{Global Conditioning in Spatial Methods:} Spatial coupling 
    techniques \cite{chazottes2007concentration} offer powerful tools for 
    achieving optimal $\mathcal{O}(N)$ bounds on undirected graphical models 
    and Markov random fields \cite{wainwright2008graphical}. These methods, 
    however, are designed for a different problem setting. They operate by 
    establishing concentration via differing spatial boundary conditions on 
    a global Gibbs measure \cite{georgii2011gibbs}, an approach that inherently 
    requires conditioning on the entire graphical structure simultaneously. 
    This global conditioning renders them inapplicable under the strictly 
    causal filtration required by autoregressive generation, where each step 
    must depend solely on the past.
\end{enumerate}

The MDC framework resolves this tension. It operates within the standard causal sequential martingale setting---requiring no information about the future. By defining the interdependence matrix $H$ via one-step conditional transitions, MDC natively preserves the directed $d$-separation of the causal graph \cite{pearl1988probabilistic}, thereby achieving an optimal $\mathcal{O}(N)$ bound for a generative process where classical spatial methods are inapplicable. We formalize the topological and generative assumptions:
\begin{itemize}[leftmargin=*]
    \item \textbf{Topological Skeleton:} The index set $V = \{1, \dots, N\}$ forms a directed forest under a valid topological sort. Each non-root node $j$ has a unique parent $\pi(j)$ satisfying $\pi(j) < j$. The maximum out-degree of the graph is bounded by an integer $D \ge 1$.
    \item \textbf{Generative Kernel:} The conditional transition distribution of node $j$ depends exclusively on its immediate parent. For any historical prefix $x_{1:j-1}$, the transition kernel satisfies $p_j(\cdot \mid x_{1:j-1}) = P_j(\cdot \mid x_{\pi(j)})$. 
    \item \textbf{Sub-critical Contraction:} For any node $j$, the single-step conditional divergence satisfies $H_{\pi(j), j} \le \alpha$. We assume the tree operates in the sub-critical regime, satisfying $\alpha D < 1$.
\end{itemize}

Under these assumptions, we evaluate the entries of the interdependence matrix $H$. By definition \eqref{defH}, computing $H_{i,j}$ requires taking the supremum over prefixes that differ exclusively at coordinate $i$. Because the topological sort guarantees $\pi(j) < j$, the parent state $x_{\pi(j)}$ is strictly contained within the conditioning history $x_{1:j-1}$. For any perturbation at coordinate $i \neq \pi(j)$, the parent state $x_{\pi(j)}$ remains mathematically identical in both evaluated histories. Consequently, the local transition kernel $p_j$ outputs identical probability measures, yielding exactly $H_{i,j} = 0$. 

This algebraic property proves that $H$ is structurally sparse, bounded exactly by the adjacency matrix of the directed tree scaled by $\alpha$. Let $d(i,j)$ denote the directed path length from ancestor $i$ to descendant $j$. The resolvent matrix $\itGamma = (I-H)^{-1} = \sum_{r=0}^{N-1} H^r$ evaluates exactly to $\itGamma_{i,j} \le \alpha^{d(i,j)} \mathbbm{1}_{\{i \text{ is an ancestor of } j\}}$.

\begin{proposition}[Concentration on Sub-critical Causal Trees]
\label{prop:causal_trees}
Let $S_N = \sum_{i=1}^N f_i(X_i)$ be an additive target function with uniform coordinate-wise Lipschitz constants $c_i = 1$, $i=1,\dots,N$. Then the exact matrix-vector multiplication guarantees the uniform squared vector $\ell_2$-norm bound $\|\itGamma \mathbf{1}_N\|_2^2 \le \frac{N}{(1 - \alpha D)^2}$. Consequently, the sequence satisfies:
\beqlb\label{eq:branching_concentration}
\mathbb{P}\bigl(|S_N-\mathbb{E}[S_N]|\ge t\bigr)
\le
2\exp\left(-\frac{2t^{2}(1 - \alpha D)^2}{N}\right).
\eeqlb
\end{proposition}

\begin{proof}
Because $\itGamma_{i,j}$ is non-zero exclusively for valid directed paths, and the maximum number of descendants at generation distance $r$ is bounded by $D^r$, the row sum of $\itGamma$ is bounded by a convergent geometric series under the assumption $\alpha D < 1$:
$$
(\itGamma \mathbf{1}_N)_i = \sum_{j=i}^N \itGamma_{i,j} \le \sum_{r=0}^{\infty} D^r \alpha^r = \frac{1}{1 - \alpha D}.
$$
Substituting this uniform row-sum bound into the definition of the squared $\ell_2$-norm yields the variance proxy $\|\itGamma \mathbf{1}_N\|_2^2 = \sum_{i=1}^N (\itGamma \mathbf{1}_N)_i^2 \le \frac{N}{(1 - \alpha D)^2}$. Applying Theorem \ref{thm:general_matrix_decoupling} yields the result.
\end{proof}

\section{Application to Large Language Models}
\label{sec:llm_application}

Classical concentration inequalities relying on global scalar decay coefficients typically assume mutual independence or macroscopic mixing. These assumptions are mathematically violated in autoregressive Large Language Models (LLMs), where the self-attention mechanism introduces dense, non-Markovian dependencies across the sequence. In this section, we apply the MDC framework to establish finite-sample deviation bounds for sequence-level evaluations in LLMs.

\subsection{Mapping Autoregressive Generation and Sequence-Level Targets}

Consider an autoregressive language model generating a sequence of categorical tokens $\mathbf{X} = (X_1, \dots, X_N)$. To apply the MDC framework, we explicitly define the mapping between the mathematical abstraction of Section \ref{sec:setup} and the generative process:
\begin{itemize}[leftmargin=*]
    \item \textbf{State Space ($\mathcal{A} \to \mathcal{V}$):} The abstract finite state space $\mathcal{A}$ corresponds to the token vocabulary $\mathcal{V}$. The generation follows the step-by-step conditional probability $\mathbb{P}(X_j \in \cdot \mid X_{1:j-1})$.
    \item \textbf{Target Function ($f \to \text{Reward}$):} The measurable target function $f: \mathcal{V}^N \to \mathbb{R}$ represents the sequence-level reward model utilized in Reinforcement Learning from Human Feedback (RLHF) \cite{christiano2017deep, ouyang2022training}. Establishing theoretical guarantees for policy evaluation requires bounding the finite-sample deviation $\mathbb{P}(|f(\mathbf{X}) - \mathbb{E}[f(\mathbf{X})]| \ge t)$.
    \item \textbf{Sensitivity Vector ($\mathbf{c}$):} The deterministic sensitivity of the target function $f$ is bounded by Condition \ref{cond:lipschitz_general}. If modifying the $j$-th token alters the scalar evaluation $f$ by at most $c_j$, the sensitivity vector is $\mathbf{c} = (c_1, \dots, c_N)^\top$. In tasks evaluating strictly the final generated answer, modifying intermediate tokens does not deterministically alter the terminal score. This structural property yields zero entries in $\mathbf{c}$ for all non-terminal coordinates.
    \item \textbf{Interdependence Matrix ($H$):} The matrix $H$ quantifies directed causal influence within the autoregressive decoding process. The entry $H_{i,j}$ bounds the total variation shift in the local transition kernel $p_j$ induced by altering the input token at step $i$, while conditioning on a fixed intermediate context. This mathematically models the attention weights and hidden state dynamics of the sequence model.
\end{itemize}

To ensure the mathematical convergence of the variance proxy, we specify the following condition on the cumulative causal dependence.

\begin{condition}[Sub-critical Causal Influence]\label{cond:subcritical_attention}
The interdependence matrix $H$ satisfies the maximum absolute column-sum bound $\|H\|_1 = \max_j \sum_{i=1}^{j-1} H_{i,j} \le \alpha$ for a constant $\alpha \in (0, 1)$.
\end{condition}

This mathematical condition aligns with architectural constraints in efficient long-context generation. In sliding window attention \cite{beltagy2020longformer, jiang2023mistral}, the direct causal influence evaluates strictly to zero outside a fixed local window, bounding the column sum irrespective of the sequence length $N$. In state space models \cite{gu2023mamba} and linear transformers \cite{katharopoulos2020transformers}, the recurrent hidden states are constrained to be contractive operators. This ensures an exponential decay in the total variation shift over time, guaranteeing the convergence of the column sums.

Under Condition \ref{cond:subcritical_attention}, the resolvent matrix $\itGamma = (I - H)^{-1}$ exists. By the standard property of induced matrix norms, its $\ell_1$-operator norm is strictly bounded: $\|\itGamma\|_1 \le \frac{1}{1-\alpha}$.

\subsection{Dimension-Free Concentration for Sparse Terminal Targets}

The explicit evaluation of the exact matrix-vector multiplication $\|\itGamma \mathbf{c}\|_2^2$ provides an analytical mechanism to prevent scalar collapse. We demonstrate this by analyzing a target function $f$ evaluated exclusively on the terminal state. The corresponding sensitivity vector contains a single non-zero entry: $\mathbf{c} = (0, \dots, 0, c_N)^\top$.

We apply Theorem \ref{thm:general_matrix_decoupling} to derive a concentration bound that scales strictly independently of the sequence length $N$.

\begin{figure}[htbp]
    \centering
    \includegraphics[width=0.7\textwidth]{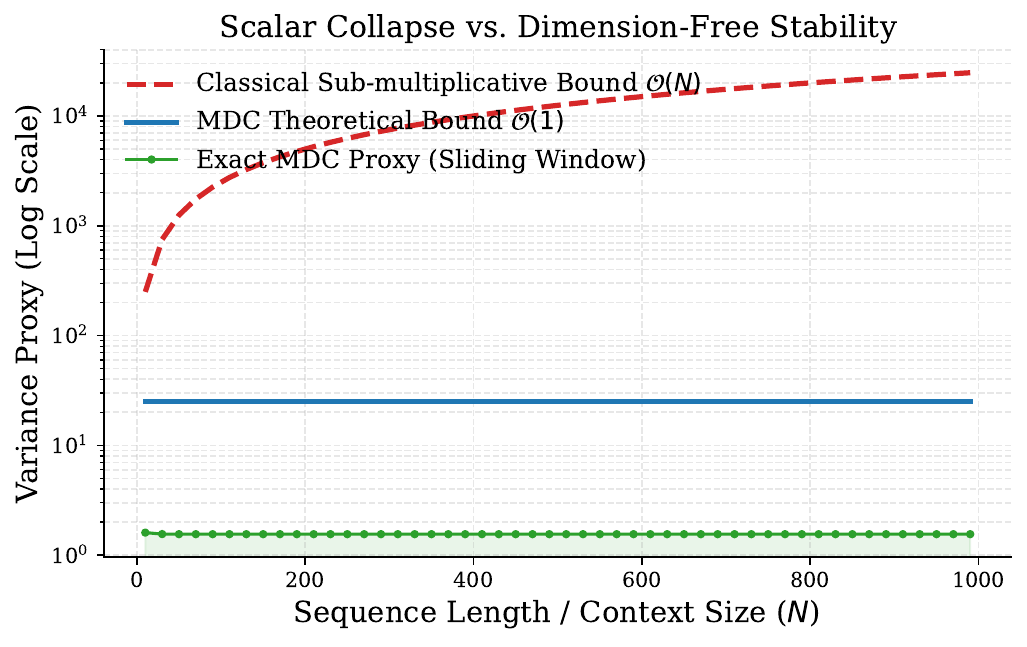}
    \caption{\textbf{Numerical Evaluation of Variance Proxy Bounds.} Theoretical variance proxies are evaluated for a non-Markovian autoregressive sequence with sliding window attention (window size $W=5$, maximum column-sum $\alpha=0.8$). For a sparse terminal target function ($\mathbf{c} = (0, \dots, 0, c_N)^\top$), classical macroscopic relaxation ($\sigma^2 \le N \|\itGamma\|_{\infty}^2 \|\mathbf{c}\|_\infty^2$) decouples the sensitivity vector from the resolvent matrix, structurally forcing an $\mathcal{O}(N)$ scaling divergence. In contrast, evaluating the exact matrix-vector multiplication $\|\itGamma \mathbf{c}\|_2^2$ under the MDC framework preserves the coordinate-wise zero entries of $\mathbf{c}$, mathematically isolating the terminal variation. This strictly maintains a dimension-free $\mathcal{O}(1)$ bound, verifying the theoretical guarantees established in Proposition \ref{prop:llm_reward}.}
    \label{fig:mdc_comparison}
\end{figure}

\begin{proposition}[Dimension-Free Bound for Sparse Targets]\label{prop:llm_reward}
Let $\mathbf{X}$ be an autoregressive sequence generated by a system satisfying Condition \ref{cond:subcritical_attention}. Let $f(\mathbf{X})$ be a target function evaluated solely on the final token, corresponding to the sparse sensitivity vector $\mathbf{c} = (0, \dots, 0, c_N)^\top$. The evaluation satisfies the following concentration inequality:
\beqnn
\mathbb{P}\bigl(|f(\mathbf{X}) - \mathbb{E}[f(\mathbf{X})]|\ge t\bigr)
\le
2\exp\left(-\frac{2t^{2} (1-\alpha)^2}{c_N^2}\right).
\eeqnn
\end{proposition}

\begin{proof}
By Theorem \ref{thm:general_matrix_decoupling}, the variance proxy is determined by $\|\itGamma \mathbf{c}\|_2^2$. Because $\mathbf{c} = (0, \dots, 0, c_N)^\top$, the exact matrix-vector multiplication exactly isolates the $N$-th column of $\itGamma$. Since $H$ is entry-wise non-negative, the Neumann series $\itGamma = \sum_{r=0}^{N-1} H^r$ is also non-negative. Algebraically bounding the sum of squares by the square of the sum yields:
$$ \|\itGamma \mathbf{c}\|_2^2 = \sum_{k=1}^N (\itGamma_{k, N} c_N)^2 \le c_N^2 \left( \sum_{k=1}^N \itGamma_{k,N} \right)^2 \le c_N^2 \|\itGamma\|_1^2. $$
By the property of the induced $\ell_1$-norm under Condition \ref{cond:subcritical_attention}, we have $\|\itGamma\|_1 \le (1-\alpha)^{-1}$. Thus, $\|\itGamma \mathbf{c}\|_2^2 \le c_N^2 (1-\alpha)^{-2}$. This proxy is strictly independent of the sequence length $N$, which completes the proof.
\end{proof}

\begin{remark}[The Algebraic Mechanism against Scalar Collapse]
Classical sequential martingale inequalities \cite{kontorovich2008concentration} decouple the dependency matrix from the sensitivity vector by applying a sub-multiplicative norm inequality, bounding the variance proxy via the operator infinity norm: $\sigma^2 \le N \|\itGamma\|_{\infty}^2 \|\mathbf{c}\|_\infty^2$. For a sparse target where $\|\mathbf{c}\|_\infty = c_N$, this macroscopic relaxation evaluates to $\mathcal{O}(N)$. Algebraically, this sub-multiplicative decoupling treats the maximum coordinate sensitivity as uniformly applicable across all dimensions, artificially inflating the theoretical variance to scale linearly with the context length. 

In the MDC framework, evaluating the exact matrix-vector multiplication $\itGamma \mathbf{c}$ prior to applying the vector norm structurally prevents this inflation. As derived in the proof of Proposition \ref{prop:llm_reward}, the zero entries in $\mathbf{c}$ multiply with the off-diagonal entries of $\itGamma$, mathematically nullifying their contribution to the resulting variance proxy. This algebraic operation strictly shifts the bounding constraint from the matrix row sum ($\|\itGamma\|_\infty$) to the matrix column sum ($\|\itGamma\|_1$). This exact algebraic preservation provides a formal mathematical proof for the theoretical stability of autoregressive models generating long-context sequences under sparse sequence-level evaluations.
\end{remark}

\section{Conclusion and Future Work}

In this work, we introduced the Matrix-Decoupled Concentration (MDC) framework for dependent sequences. Crucially, our variance proxy $\|\itGamma \mathbf{c}\|_2^2$ explicitly retains the full structural characteristics of both the target function's sensitivity vector and the system's dependency structure, circumventing classical scalar collapse.

The main result, Theorem \ref{thm:general_matrix_decoupling}, allows us to: (i) recover the absolute optimal constants for contracting Markov chains, (ii) establish optimal $\mathcal{O}(N)$ bounds for directed causal trees, and (iii) yield dimension-free $\mathcal{O}(1)$ guarantees for densely connected autoregressive LLMs under sparse terminal rewards. By breaking traditional dimensional barriers, MDC provides a rigorous theoretical justification for the stability of long-context generation in modern AI.

While our method resolves the scalar collapse problem for bounded differences, it opens several foundational avenues for future research:
\begin{enumerate}[leftmargin=*]
    \item \textbf{Predictable Quadratic Variation (Freedman-Type Bounds):} 
    Our current framework fundamentally relies on the worst-case absolute span $\delta_k$ via Hoeffding's Lemma. In highly sparse transition systems, the actual probability of a structural shift may be microscopically small. Developing a \textit{Freedman/Bernstein-type martingale bound} \cite{freedman1975on, tropp2012user} that tracks the propagation of data-dependent variances via explicit matrix-vector multiplication would provide a massive leap in tightness for heavy-tailed dependent systems.
    
    \item \textbf{Beyond Worst-Case (Annealed Average Coupling):} 
    The interdependence matrix $H_{i,j}$ takes a stringent supremum over all possible past trajectories (a \textit{quenched} $L^\infty$ bound). Relaxing the definition of $H_{i,j}$ to an \textit{annealed} $L^1$ bound---averaging over the natural probability measure of the history rather than taking the supremum---would vastly expand the applicability to non-uniformly mixing sequences.
    
    \item \textbf{Beyond Directed Time (Markov Random Fields):} 
    Our sequential approach intrinsically relies on the Doob martingale decomposition, enforcing a unidirectional arrow of time. Investigating whether the resolvent matrix bound strictly holds for non-triangular dependency graphs via Glauber dynamics represents a fascinating frontier connecting empirical process theory with statistical physics \cite{levin2017markov}.
\end{enumerate}

\bibliographystyle{unsrtnat}
\bibliography{ai_arxiv}

\section*{Appendix: Proof of Theorem \ref{thm:general_matrix_decoupling}}\label{sec:appendix_proof}

\begin{proof}
Let $(\mathcal{F}_i)_{i=0}^N$ be the natural filtration generated by $\{X_i\}_{i=1}^N$, with $\mathcal{F}_0 = \{\emptyset, \Omega\}$. Define the Doob martingale $M_k = \mathbb{E}[f(\mathbf{X})\mid \mathcal{F}_k]$ for $k = 0, \dots, N$. We have $M_0 = \mathbb{E}[f(\mathbf{X})]$ and $M_N = f(\mathbf{X})$. By telescoping, 
\beqnn
f(\mathbf{X})- \mathbb{E}[f(\mathbf{X})] = \sum_{k=1}^N \Delta_k, \quad \text{where} \quad \Delta_k = M_k - M_{k-1}. 
\eeqnn
Using the localized transition probability kernels $p_j(\cdot \mid x_{1:j-1})$, we define the integration function for any history $\mathbf{u} \in \mathcal{A}^k$:
\beqlb\label{defF}
F_k(\mathbf{u}) := \sum_{x_{k+1:N} \in \mathcal{A}^{N-k}} f(\mathbf{u}, x_{k+1:N}) \prod_{j=k+1}^N p_j(x_j \mid \mathbf{u}, x_{k+1:j-1}).
\eeqlb
Conditional on the event $\{X_{1:k}=\mathbf{u}\}$, we have $M_k = F_k(\mathbf{u})$ almost surely.

Fix the past filtration $\mathcal{F}_{k-1}$. Conditioned on any valid realization $\{X_{1: k-1}=\mathbf{x}\} \in \mathcal{F}_{k-1}$ with strictly positive probability, the predictable term $M_{k-1}$ is deterministically fixed, and the martingale $M_k$ takes values depending exclusively on the realization of $X_k$. Thus, the maximum conditional range (oscillation) of $\Delta_k$ is bounded exactly by:
\beqnn
\delta_k(\mathbf{x}) \ar:=\ar 
\max_{x, x' \in \mathcal{A}} \Bigl| F_k(\mathbf{x}, x) - F_k(\mathbf{x}, x') \Bigr|.
\eeqnn

Fix the worst-case pair $(x, x')$ and the valid prefix $\mathbf{x}$.
We sequentially construct a joint probability coupling $(\mathbf{Y}, \mathbf{Z}) \in \mathcal{A}^N \times \mathcal{A}^N$ on an extended probability space $(\hat{\Omega}, \hat{\mathcal{F}}, \hat{\mathbb{P}})$, equipped with the natural filtration $\hat{\mathcal{F}}_j = \sigma(Y_{1:j}, Z_{1:j})$. For the initial steps up to $k$, we deterministically set $Y_{1:k} = (\mathbf{x}, x)$ and $Z_{1:k} = (\mathbf{x}, x')$.
 
Proceeding by induction, for any step $j > k$, suppose that the filtered probability space and the coupling process $(\mathbf{Y}, \mathbf{Z})$ have been validly defined up to step $j-1$. Let $\{Y_{k:j-1}=\mathbf{y}, Z_{k:j-1}=\mathbf{z}\}$ be an arbitrary atom of $\hat{\mathcal{F}}_{j-1}$. Conditioned on this event and the fixed prefix $X_{1:k-1}=\mathbf{x}$, we define the target conditional marginal distributions for step $j$ via the transition kernels of $\mathbf{X}$: 
\beqlb\label{eq:target_mu}
\mu_j^+(\cdot) \ar:=\ar p_j(\cdot \mid \mathbf{x}, \mathbf{y}), \cr
\mu_j^-(\cdot) \ar:=\ar p_j(\cdot \mid \mathbf{x}, \mathbf{z}).
\eeqlb

The Optimal Coupling Theorem for Total Variation (see, e.g., \cite{levin2017markov}) guarantees that there exist random variables $(Y_j, Z_j)$ such that almost surely:
\beqlb\label{eq:coupled_prob}
\hat{\mathbb{P}}(Y_j \in \cdot \mid \hat{\mathcal{F}}_{j-1}) \ar=\ar \mu_j^+(\cdot), \cr
\hat{\mathbb{P}}(Z_j \in \cdot \mid \hat{\mathcal{F}}_{j-1}) \ar=\ar \mu_j^-(\cdot),
\eeqlb
and their conditional probability of disagreement almost surely attains the total variation lower bound:
\beqlb\label{div1}
\hat{\mathbb{P}}(Y_j \neq Z_j \mid \hat{\mathcal{F}}_{j-1}) \ar=\ar d_{\mathrm{TV}}(\mu_j^+, \mu_j^-).
\eeqlb
Iterating this optimal conditional coupling up to the final horizon $N$ completes the construction of the joint trajectories $(\mathbf{Y}, \mathbf{Z})$. 

To bound the conditional divergence $d_{\mathrm{TV}}(\mu_j^+, \mu_j^-)$, we employ a path coupling argument. We construct $\hat{\mathcal{F}}_{j-1}$-measurable hybrid vectors to interpolate between the coupled histories. For each $m \in \{k-1, \dots, j-1\}$, define:
\beqnn
W^{(m)} = (Z_k, \dots, Z_m, Y_{m+1}, \dots, Y_{j-1}),
\eeqnn
with the boundary conventions $W^{(k-1)} = Y_{k:j-1}$ and $W^{(j-1)} = Z_{k:j-1}$. Applying the triangle inequality of the metric, we evaluate the telescoping sum:
\beqlb\label{eq:telescoping_tv}
d_{\mathrm{TV}}(\mu_j^+, \mu_j^-) \le \sum_{m=k}^{j-1} d_{\mathrm{TV}}\Big( 
p_j(\cdot \mid \mathbf{x}, W^{(m)}), p_j(\cdot \mid \mathbf{x}, W^{(m-1)}) \Big). 
\eeqlb
Because $W^{(m)}$ and $W^{(m-1)}$ differ exclusively at the $m$-th coordinate (taking values $Z_m$ and $Y_m$, respectively), their conditional TV distance evaluates exactly to $0$ if $Y_m = Z_m$. If $Y_m \neq Z_m$, definition \eqref{defH} bounds this distance by $H_{m,j}$. Therefore:
\beqnn
\hat{\mathbb{P}}(Y_j \neq Z_j \mid \hat{\mathcal{F}}_{j-1}) \le \sum_{m=k}^{j-1} H_{m,j} \mathbbm{1}_{\{Y_m \neq Z_m\}}.
\eeqnn

For any $i=1,\dots, N$, define the marginal discrepancy probability $v_i := \hat{\mathbb{P}}(Y_i \neq Z_i)$. The deterministic initialization implies $v_m = 0$ for all $m < k$, and for the worst-case pair $x \neq x'$, exactly $v_k = 1$. Taking the unconditional expectation $\hat{\mathbb{E}}[\cdot]$ via the tower property yields:
\beqnn
v_j = \hat{\mathbb{E}}\Big[ \hat{\mathbb{P}}(Y_j \neq Z_j \mid \hat{\mathcal{F}}_{j-1}) \Big] \le \sum_{m=k}^{j-1} H_{m,j} \hat{\mathbb{E}}\big[\mathbbm{1}_{\{Y_m \neq Z_m\}}\big] = \sum_{m=k}^{j-1} H_{m,j} v_m, \quad \forall j > k.
\eeqnn
Define the discrepancy vector $\mathbf{v} := (0, \dots, 0, 1, v_{k+1}, \dots, v_N)^\top$ and let $\mathbf{e}_k$ be the $k$-th standard basis vector. The recursive inequalities reformulate algebraically as $(I - H^\top)\mathbf{v} \le \mathbf{e}_k$. Because $H$ is entry-wise non-negative, the transposed resolvent matrix $\itGamma^\top = (I - H^\top)^{-1} = \sum_{r=0}^{N-1} (H^\top)^r$ exists and is entry-wise non-negative. Left-multiplying the inequality by $\itGamma^\top$ strictly preserves the direction of the inequality, yielding:
\beqlb\label{inecop1}
\mathbf{v} \le (I - H^\top)^{-1} \mathbf{e}_k = \itGamma^\top \mathbf{e}_k.
\eeqlb

Since the coupling deterministically sets $Y_{1:k} = (\mathbf{x}, x)$ and $Z_{1:k} = (\mathbf{x}, x')$, sequentially applying the chain rule of probability to \eqref{eq:target_mu} and \eqref{eq:coupled_prob} yields:
\beqnn
\hat{\mathbb{P}}(Y_{k+1:N} = y_{k+1:N}) \ar=\ar \prod_{j=k+1}^N p_j(y_j \mid \mathbf{x}, x, y_{k+1:j-1}).
\eeqnn
This exactly matches the product measure of the integration function $F_k$ defined in \eqref{defF}, which strictly implies that $\hat{\mathbb{E}}[f(\mathbf{Y})] = F_k(\mathbf{x}, x)$. By identical reasoning for the sequence $\mathbf{Z}$ with prefix $(\mathbf{x}, x')$, we have $\hat{\mathbb{E}}[f(\mathbf{Z})] = F_k(\mathbf{x}, x')$. The martingale difference equals the expected divergence under the coupling:
\beqnn
F_k(\mathbf{x}, x) - F_k(\mathbf{x}, x') = \hat{\mathbb{E}}\big[ f(\mathbf{Y}) - f(\mathbf{Z}) \big].
\eeqnn

Applying Condition \ref{cond:lipschitz_general} to the coupled trajectories yields $|f(\mathbf{Y}) - f(\mathbf{Z})| \le \sum_{j=1}^N c_j \mathbbm{1}_{\{Y_j \neq Z_j\}}$. Taking the expectation $\hat{\mathbb{E}}[\cdot]$ on both sides and utilizing $\mathbbm{1}_{\{Y_j \neq Z_j\}} = 0$ for all $j < k$, we obtain:
\beqnn
\Bigl| \hat{\mathbb{E}}[f(\mathbf{Y}) - f(\mathbf{Z})] \Bigr| 
\le \hat{\mathbb{E}}\Big[ \big| f(\mathbf{Y}) - f(\mathbf{Z}) \big| \Big] 
\le \hat{\mathbb{E}}\left[ \sum_{j=k}^N c_j \mathbbm{1}_{\{Y_j \neq Z_j\}} \right] 
= \sum_{j=k}^N c_j v_j = \mathbf{c}^\top \mathbf{v},
\eeqnn
where the final inequality holds because $v_k = 1 \ge \mathbbm{1}_{\{Y_k \neq Z_k\}}$. Substituting the derived vector inequality \eqref{inecop1} evaluates to:
\beqnn
\mathbf{c}^\top \mathbf{v} \le \mathbf{c}^\top (\itGamma^\top \mathbf{e}_k) = (\itGamma \mathbf{c})^\top \mathbf{e}_k = (\itGamma \mathbf{c})_k.
\eeqnn
Since this holds for any valid prefix $\mathbf{x}$ and any state pair $(x, x')$, the conditional range of the martingale difference $\Delta_k$ is strictly bounded by length $\delta_k(\mathbf{x}) \le (\itGamma \mathbf{c})_k$.

Conditional on $\mathcal{F}_{k-1}$, the zero-mean random variable $\Delta_k$ lies almost surely in an interval of length at most $(\itGamma \mathbf{c})_k$. Applying Hoeffding's Lemma bounds the conditional moment generating function for any $\lambda > 0$:
\beqnn
\mathbb{E}[\exp(\lambda \Delta_k) \mid \mathcal{F}_{k-1}] \le \exp\left( \frac{\lambda^2 \delta_k(\mathbf{x})^2}{8} \right) \le \exp\left( \frac{\lambda^2 (\itGamma \mathbf{c})_k^2}{8} \right).
\eeqnn

Applying the tower property of conditional expectation recursively, the moment generating function of the centered sum $S_N = f(\mathbf{X}) - \mathbb{E}[f(\mathbf{X})] = \sum_{k=1}^N \Delta_k$ evaluates to:
\beqnn
\mathbb{E}[\exp(\lambda S_N)] \ar=\ar \mathbb{E}\Big[ \exp(\lambda S_{N-1}) \mathbb{E}[ \exp(\lambda \Delta_N) \mid \mathcal{F}_{N-1} ] \Big] \cr
\ar\le\ar \dots \le \exp\left( \frac{\lambda^2}{8} \sum_{k=1}^N (\itGamma \mathbf{c})_k^2 \right) = \exp\left( \frac{\lambda^2 \|\itGamma \mathbf{c}\|_2^2}{8} \right).
\eeqnn

Applying Markov's inequality to the exponential moments yields the Chernoff bound:
\beqnn
\mathbb{P}(S_N \ge t) \le \inf_{\lambda > 0} \exp(-\lambda t) \mathbb{E}[\exp(\lambda S_N)] \le \inf_{\lambda > 0} \exp\left( -\lambda t + \frac{\lambda^2 \|\itGamma \mathbf{c}\|_2^2}{8} \right).
\eeqnn
Optimizing the quadratic exponent by setting $\lambda = 4t / \|\itGamma \mathbf{c}\|_2^2$ yields the one-sided bound $\exp(-2t^2 / \|\itGamma \mathbf{c}\|_2^2)$. A symmetric derivation for $-S_N$ guarantees the two-sided concentration inequality \eqref{eq:delta_matrix_bound_general}. 
\end{proof}

\begin{proof}[Proof of Corollary \ref{cor:general_kappa}]
It is clear from the definition of the induced matrix norm that $\|\itGamma \mathbf{c}\|_2^2 \le \|\itGamma\|_2^2 \|\mathbf{c}\|_2^2 = \frac{1}{\kappa} \|\mathbf{c}\|_2^2$. 
For the uniform decay case, Schur's test guarantees that  $\|H\|_2 \le \sqrt{\|H\|_1 \|H\|_\infty} \le S$. 
Using the discrete Neumann series, we have $\|\itGamma\|_2 \le (1 - \|H\|_2)^{-1} \le (1 - S)^{-1}$, which directly yields the lower bound $\kappa \ge (1 - S)^2$. 
Applying Theorem \ref{thm:general_matrix_decoupling} with this variance proxy completes the proof.
\end{proof}

\end{document}